\documentclass[letterpaper, 10 pt, conference]{ieeeconf}
\IEEEoverridecommandlockouts                              %

\usepackage{times}

\makeatletter
\let\NAT@parse\undefined
\makeatother
\usepackage[numbers]{natbib}
\usepackage{multicol}
\usepackage[bookmarks=true]{hyperref}

\usepackage[]{graphicx}
\usepackage{enumerate}
\usepackage{subfigure}

\usepackage{amssymb}
\usepackage{amsthm}

\graphicspath{{./fig/}}
\usepackage{color}
\usepackage{amsmath}
\usepackage{amssymb}
\usepackage{graphicx}
\usepackage{comment,xspace}
\usepackage{fancybox}
\usepackage{xspace}

\newtheorem{theorem}{Theorem}[section]

\newtheorem{lemma}[theorem]{Lemma}
\newtheorem{corollary}[theorem]{Corollary}
\newtheorem{remark}[theorem]{Remark}

\makeatletter
\newtheorem*{rep@theorem}{\rep@title}
\newcommand{\newreptheorem}[2]{%
\newenvironment{rep#1}[1]{%
 \def\rep@title{#2 \ref{##1}}%
 \begin{rep@theorem}}%
 {\end{rep@theorem}}}
\makeatother

\newreptheorem{theorem}{Theorem}
\newreptheorem{lemma}{Lemma}

\newcommand{\real}{{\mathbb{R}}}
\newcommand{\reals}{\real}

\newcommand{\xgoal}{\mathcal M_{\text{goal}}}

\newcommand{\eps}{\varepsilon}

\newcommand{\Span}{\operatorname{Span}}

\newcommand{\argmin}{\operatornamewithlimits{argmin}}

\newcommand{\p}[1]{\mbox{$\mathbb{P}\left(#1\right)$}} 
 
\newcommand{\probcond}[2]{\mbox{$\mathbb{P}\left(#1 \,| \, #2\right)$}}

\newcommand{\LL}{\mathcal{L}}
\newcommand{\M}{\mathcal{M}}
\newcommand{\Mobs}{\mathcal{M}_{\text{obs}}}
\newcommand{\Mfree}{\mathcal{M}_{\text{free}}}
\newcommand{\Mgoal}{\mathcal{M}_{\text{goal}}}

\newcommand{\cl}{\operatorname{cl}}

\newcommand{\SF}{\texttt{SampleFree}}
\newcommand{\R}{\mathbb{R}}
\newcommand{\N}{\mathbb{N}}

\newcommand{\amin}{a_{\min}}
\newcommand{\Amax}{A_{\max}}
\newcommand{\thmin}{\theta_{\min}}
\newcommand{\sigmamin}{\sigma_{\min}}
\newcommand{\rank}{s}
\newcommand{\xinit}{x_{\text{init}}}
\newcommand{\poly}{\text{poly}}

\newcommand{\boxw}{\text{Box}^w}
\newcommand{\nti}{n\to\infty}

\newcommand{\X}{\mathcal{X}}
\newcommand{\DFMT}{DFMT$^*$\xspace}
\newcommand{\DPRM}{DPRM$^*$\xspace}
\newcommand{\FMT}{FMT$^*$\xspace}
\newcommand{\PRM}{PRM$^*$\xspace}

\newcommand{\RRTs}{RRT$^*$\xspace}

\usepackage{algorithmic}
\usepackage{algorithm}

\usepackage{subfigure}

\title{Optimal Sampling-Based Motion Planning under\\ Differential Constraints: the Driftless Case}

\author{Edward Schmerling, Lucas Janson, Marco Pavone%
	\thanks{Edward Schmerling is with the Institute for Computational \& Mathematical \ Engineering, Stanford University, Stanford, CA 94305, \texttt{schmrlng@stanford.edu}.}
	\thanks{Lucas Janson is with the Department  of Statistics, Stanford University, Stanford, CA 94305, \texttt{ljanson@stanford.edu}.}
	\thanks{Marco Pavone is with the Department\ of Aeronautics and Astronautics, Stanford University, Stanford, CA 94305, \texttt{pavone@stanford.edu}.}
	\thanks{This work was supported by an Early Career Faculty grant from NASA's Space Technology Research Grants Program (Grant NNX12AQ43G).}
}

\begin{document}

\maketitle

\begin{abstract}
Motion planning under differential constraints is a classic  problem in robotics. To date, the state of the art is represented by sampling-based techniques, with the Rapidly-exploring Random Tree algorithm as a leading example. Yet, the problem is still open in many aspects, including guarantees on the quality of the obtained solution. In this paper we provide a thorough theoretical framework to assess optimality guarantees of sampling-based algorithms for planning under differential constraints. We exploit this framework to design and analyze two novel sampling-based algorithms that are guaranteed to converge, as the number of samples increases, to an optimal solution (namely, the Differential Probabilistic RoadMap algorithm and the Differential Fast Marching Tree algorithm). Our focus is on driftless control-affine dynamical models, which accurately model a large class of robotic systems.  In this paper we use the notion of convergence in probability (as opposed to convergence almost surely): the extra mathematical flexibility of this approach yields convergence rate bounds --- a first in the field of optimal sampling-based motion planning under differential constraints. Numerical experiments corroborating our theoretical results are presented and discussed.
\end{abstract}

\IEEEpeerreviewmaketitle

\section{Introduction}\label{sec:intro}

Motion planning is a fundamental problem in robotics. It involves the computation of a sequence of actions that drives a robot from an initial condition to a terminal condition while avoiding obstacles, respecting kinematic/dynamical constraints, and possibly optimizing an objective function \citep{Lavalle:06}. The basic problem, where a robot does not have any constraints on its motion and only an obstacle-free solution is required, is well-understood and solved for a large number of practical scenarios \cite{Lavalle:RAM11}. On the other hand, robots do usually have stringent kinematic/dynamical (in short, differential) constraints on their motion, which in most settings need to be properly taken into account. There are two main approaches \cite{Lavalle:RAM11}: (i) a decoupling approach, in which the problem is decomposed in steps of computing a geometric collision-free path (neglecting the differential constraints), smoothing the path to satisfy the motion constraints, and finally reparameterizing the trajectory so that the robot can execute it, or (ii) a direct approach, in which the differentially-constrained motion planning problem (henceforth referred to as DMP problem)  is solved in one shot. The first approach, while fairly common in practice, has several drawbacks, including the computation of very inefficient trajectories, failure in finding a trajectory due to the decoupling scheme itself, and inflated information requirements \cite{Lavalle:RAM11}. This motivates a quest for efficient algorithms that \emph{directly} solve the DMP  problem.

However, directly finding a feasible, let alone optimal, solution to the DMP problem is difficult (note that the basic version without differential constraints is already PSPACE-hard \cite{Reif:FCS79, Lavalle:06}, which implies NP-hard). Early work on this topic dates back to more than two decades ago \cite{Donald:ACM91}, but the problem, especially when optimality is taken into account, is still open in many aspects \cite{Tsianos.ea:CSR07, Lavalle:RAM11}, including algorithms with practical convergence rates, guarantees on the quality of the obtained solution, and class of dynamical systems that can be addressed. To date, the state of the art is represented by sampling-based techniques, where an \emph{explicit} construction of the configuration space is avoided and the configuration space is probabilistically ``probed" with a sampling scheme.

Arguably, the most successful algorithm for DMP is the sampling-based rapidly-exploring tree algorithm (RRT) \citep{LaValle.ea:IJRR01}, which incrementally builds a tree of trajectories by randomly sampling points in the configuration space. Lately, several variations of the RRT algorithm, referred to as \RRTs, stemming from \cite{SK-EF:11} and its kinodynamic extension \cite{SK-EF:10}, have been considered to ensure that the cost of the computed trajectory converges to the optimal cost as the number of sampled points goes to infinity \cite{SK-EF:10, AP-ea:2012, AP-ea:2013, SK-EF:13, Webb.ea:ICRA13}. These works, while providing strong experimental validation, only provide proof sketches that do not fully address many of the complications that arise in extending asymptotic optimality arguments from the geometric case to differentially constrained paths. For example, rewiring the RRT tree within a local volume containing (in expectation) a log fraction of previous samples is not sufficient in itself to claim optimality, as in \cite{AP-ea:2013, SK-EF:13}. Additional assumptions on trajectory approximability must be stated and verified for the differential constraints in question. Such requirements are discussed in \cite{SK-EF:10}, but it is not clear how assuming the existence of forward-reachable trajectory approximations is sufficient for a ``ball-to-ball'' proof technique that requires backward approximations as well. A different approach to asymptotically optimal planning has recently been proposed by STABLE SPARSE RRT which achieves optimality through random control propagation instead of connecting existing samples using a steering subroutine \cite{YL-ZL-KB:14}. This paper, like the \RRTs variations, follows a steering approach, although it may be considered less general, as it is our view that leveraging as much knowledge as possible of the differential constraints while planning is necessary in order to have a hope of planning in real-time.

\emph{Statement of Contributions}: The objective of this paper is to provide a theoretical framework to study optimality guarantees of sampling-based algorithms for the DMP problem, and to exploit this framework to design efficient sampling-based algorithms that are guaranteed to asymptotically converge to an optimal solution.  The focus of this paper is on driftless control-affine (DCA) dynamical systems of the form
\begin{equation*}
\dot x(t) = \sum_{i=1}^m g_i(x(t)) u_i(t),
\end{equation*}
where the available motions of trajectories $x(t)$ are linear combinations given by input control functions $u_i(t)$ and their corresponding actions at each point in space $g_i(x)$. Our work is written from the perspective of optimizing trajectory arc length, but applies also to cost metrics satisfying similar properties, which we discuss. This model is often the result of nonholonomic constraints that the kinematic variables of the system must satisfy \cite{Closkey:95}. A large class of robotic systems can be modeled as DCA, including mobile robots with wheels that roll without slipping \cite{Lavalle:06}, multi-fingered robotic hands  \cite{Murray:94}, rigid bodies with zero angular momentum undergoing re-orientation  \cite{Leonard:93}, and systems with nonholonomic actuators \cite{Closkey:95}. The DCA model, however, rules out the presence of dynamics with drift, which in many important cases (e.g., spacecraft control) can not be neglected. Specifically, the contribution of this paper is threefold. First, we show that \emph{any} trajectory of a DCA system may be ``traced" arbitrarily well by connecting randomly distributed points from a sufficiently large sample set covering the configuration space. We will refer to this property as \emph{probabilistic exhaustivity}, as opposed to probabilistic completeness \citep{Lavalle:06}, where the requirement is that \emph{at least} one trajectory is traced with a sufficiently large sample set. Probabilistic exhaustivity is a key tool in proving asymptotic optimality, and is of independent interest. Second, we introduce two novel sampling-based algorithms for the solution to the DMP with DCA systems, namely the Differential Probabilistic RoadMap algorithm (\DPRM) and the Differential Fast Marching Trees algorithm (\DFMT). Third, by leveraging the property of probabilistic exhaustivity for DCA systems, we rigorously show that both \DPRM and \DFMT are asymptotically optimal.
We note that in this paper we use the notion of  convergence in probability (as opposed to convergence almost surely, as in, e.g., \cite{SK-EF:13}): the extra mathematical flexibility of this approach yields convergence rate bounds --- a first in the field of optimal
sampling-based motion planning under differential constraints. Our approach is inspired by  \cite{SK-EF:13} and \cite{LJ-ES-AC-ea:15}.

\emph{Organization}: This paper is structured as follows. In Section \ref{sec:background} we provide a review of some key results in differential geometry that will be extensively used in the paper. In Section \ref{sec:probformulation} we formally define the problem we wish to solve, while in Section \ref{sec:exhaust} we prove the aforementioned probabilistic exhaustivity property for DCA systems. In Section \ref{sec:algo} we present the \DPRM and \DFMT algorithms, and in Section \ref{sec:AO} we prove their asymptotic optimality (together with a convergence rate characterization). In Section \ref{sec:sims} we provide implementation details for the proposed algorithms, and we study them via numerical experiments.
Finally, in Section \ref{sec:conc}, we draw some conclusions and we discuss directions for future work.

\section{Background Material}\label{sec:background}
In this section we provide some definitions and a brief review of key results in differential geometry, on which we will rely extensively
later in the paper. Let $\M \subset \R^n$ be the manifold defining a configuration space. Within this space let us consider driftless control-affine (DCA) dynamical systems of the form
\begin{equation}\label{eqn:dcasys}
\dot x(t) = \sum_{i=1}^m g_i(x(t)) u_i(t), \ \ x \in \M,\, u \in U,
\end{equation}
where the available motions of trajectories $x(t)$ are linear combinations given by input control functions $u_i(t)$ and their corresponding actions at each point in space $g_i(x)$.  We shall assume  in this paper that $g_1,\dots,g_m$ are smooth vector fields on $\M$, and that the control set $U \subset \R^m$ is closed and bounded. We also assume $U$ is symmetric about the origin so that the system is time-reversible and $0$ is in the interior of the convex hull of $U$. This last condition ensures that the local possibilities for motion at each point appear as a linear space spanned by the $g_i$, a fact essential to the forthcoming controllability discussion \cite{FB-AL:05}. We denote the driftless control-affine system of equation~\eqref{eqn:dcasys} as $\Sigma = (\M, g, U)$. A function $x:[0,T] \rightarrow \M$ is called a \emph{dynamically feasible} trajectory, alternatively path, if there exists a corresponding control function $u(t)$ with which it satisfies $\Sigma$. All trajectories discussed in this paper are dynamically feasible unless otherwise noted.

\subsection{Arc Length and Sub-Riemannian Distance}
The \emph{arc length} of a path $x(t)$ is defined as
\[
\ell(x) := \int_0^T \|\dot x(t)\| dt,
\]
where $\|\dot x(t)\| = \|\dot x(t)\|_2 = \sqrt{\langle \dot x(t), \dot x(t)\rangle}$ is computed using the standard Euclidean inner product on the tangent spaces of $\M$.
The arc length function induces a \emph{sub-Riemannian distance} $d$ on $\M$, defined for $x_1, x_2 \in \M$ as $d(x_1, x_2) := \inf_{x} \ell(x)$, where the infimum is taken over dynamically feasible trajectories $x(t)$ connecting $x_1$ and $x_2$.  Note that for driftless control-affine systems, time-reversibility implies that $d$ is symmetric and indeed a metric.  The \emph{sub-Riemannian ball} may be defined in analogy to the standard Euclidean ball (i.e., $B^e(x, \eps) = \{y \in \M : \|x-y\| \leq \eps\}$) according to $B(x, \eps) := \{y \in \M : d(x,y) \leq \eps\}$.
Note that by definition $B(x, \eps) \subset B^e(x, \eps)$.

\subsection{Controllability and Reachable Sets}
We now present a series of results regarding system controllability following the discussion in \citep{AB:96} and \cite{RM:02}. As noted above, the vector fields $g_1,\dots,g_m$ characterizing the system $\Sigma$ represent a set of possible motions for a trajectory within $\M$. More precisely at each point $p \in \M$ the vectors $g_1(p),\dots,g_m(p)$ span a linear subspace of local directions within the tangent space $T_p M$. For a vector field $Y$ on $\M$ let $\Phi_{Y,t}: \M \rightarrow \M$ denote its local flow, the function that maps an initial state to the state obtained by following the vector field $Y$ for time $t$. That is, $\Phi_{Y,t}(p) := y(t)$ where $y:[0,t]\rightarrow\M$ is a solution to the initial value problem $\dot y(\tau) = Y(y(\tau))$, $y(0) = p$. Commutators of flows, akin to control switching in \eqref{eqn:dcasys}, allow for local motions transverse to the $g_i$ to be achieved while satisfying the differential constraints. Given two vector fields $Y$, $Z$ on $\M$ and a starting point $p \in \M$, we have the approximation for small $t$:
\[
p + t^2 [Y,Z](p) = \Phi_{Y,t}\circ\Phi_{Z,t}\circ\Phi_{Y,-t}\circ\Phi_{Z,-t}(p)+ O(t^3),
\]
where the Lie bracket $[Y,Z]$ is a third vector field which may be computed with respect to a coordinate system as $[Y,Z]=J_Z X - J_X Z$, where $J_Y$ and $J_Z$ are the Jacobian matrices of $Y$ and $Z$ respectively. Computing Lie brackets allows us to characterize all directions that are possible from each $p$, in addition to those given by $g_1(p),\dots,g_m(p)$.

Let $\LL = \LL(g_1,\dots,g_m)$ be the distribution, equivalently the set of local vector subspaces, generated by the vector fields $g_1,\dots,g_m$. We define recursively $\LL^1 = \LL$,
\[]
\LL^{k+1}=\LL^k+[\LL^1,\LL^{k}]
\]
where $[\LL^1,\LL^{k}] = \Span\{[Y,Z]: Y \in \LL^1, Z \in \LL^{k}\}$.
Then $\LL^k$ is the distribution generated by the iterated Lie brackets of the $g_i$ with $k$ terms or fewer.  The Lie hull of $\LL$ is $\text{Lie}(\LL) := \cup_{k\geq 1}\LL^k$.  Let $L^k(x)$ denote the vector space corresponding to $x\in \M$ in $\LL^k$.

The vector fields $g_1,\dots,g_m$ are said to be \emph{bracket generating} if $\text{Lie}(\LL)(x) = T_x \M$ for all $x \in \M$.  This requirement is also referred to as Chow's condition, or the linear algebra rank condition, and means that arbitrary local motion may be achieved by composing motions along the control directions $g_i$. In fact, provided that the $g_i$ are bracket generating, any two points $x_1, x_2 \in \M$ may be connected by a trajectory satisfying $\Sigma$, that is $d(x_1, x_2) < \infty$. The remainder of this section develops tighter bounds on $d(x_1, x_2)$ that will be used in our asymptotic optimality proofs of planning algorithms.

Chow's condition implies that for all $x\in \M$, there exists a smallest integer $\rank = \rank(x)$ such that $L^{\rank(x)} = T_x \M$.  Indeed, we have
$
L^1(x) \subset L^2(x) \subset \dots \subset L^{\rank(x)}(x) = T_x \M.
$
Set $n_k(x) = \dim L^k(x)$.  The integer list $(n_1(x),\dots,n_{\rank(x)}(x))$ is called the growth vector of $\LL$ at $x$.  A point $x$ is called a \emph{regular point} if there exists an open neighborhood of $\M$ around $x$ such that the growth vector is constant; otherwise $x$ is said to be a \emph{singular point}.

We now further assume that every $x \in \M$ is a regular point, so that the growth vector $(n_1,\dots,n_s = n)$ is constant over the whole configuration manifold.
Fix a base point $x_0 \in \M$.  Using the bracket-generating assumption we select a local orthonormal frame for $T_{x_0} M$ of vector fields $Y_1,\dots,Y_n$ as follows: the set $\{Y_1=g_1,\dots,Y_{n_1}=g_m\}$ spans $\LL$ near $x_0$; $\{Y_1,\dots,Y_{n_2}\}$ spans $\LL^2$ near $x_0$; $\{Y_1,\dots,Y_{n_3}\}$ spans $\LL^3$ near $x_0$; and so on.  Define the weights $w_i = k$ if $Y_i(x_0) \in \LL^k(x_0)$ and $Y_i(x_0) \notin \LL^{k+1}(x_0)$. Applying a procedure developed in \citep{AB:96}, the coordinate system $y_i$ corresponding to this local frame may be transformed into a \emph{privileged} coordinate system $z_i$ by a polynomial change of coordinates of the form
\begin{equation}\label{privCOC}
{
\begin{aligned}
z_1 &= y_1,\\
z_2 &= y_2 + \poly_2(y_1),\\
z_3 &= y_3 + \poly_3(y_1, y_2),\\
\vdots\\
z_n &= y_n + \poly_n(y_1,\dots,y_{n-1}),
\end{aligned}
}
\end{equation}
where $\poly_i(\cdot)$, $i=2,\dots,n$, denotes a polynomial function that includes only terms of degree $\geq 2$ and $< w_i$.  From the triangular structure of \eqref{privCOC}, it is clear that the inverse transformation $y = z + \poly'(z)$ is of the same form.
Given privileged coordinates $z_i$, define the \emph{pseudonorm} at $x_0$ as
\[
\|z\|_{x_0} := \max\{|z_1|^{1/w_1},\dots,|z_n|^{1/w_n}\}.
\]
Using this pseudonorm we define the \emph{w-weighted box} of size $\eps$ at $x_0$ as the point set
$
\boxw(\eps) := \{z \in \R^n : \|z\|_{x_0} \leq \eps\}.
$
We use the notation $\boxw(x_0, \eps)$ for the corresponding locus of points in $\M$ given by the coordinates $z_i$.

\begin{theorem}[Ball-box theorem \citep{AB:96}]\label{thm:ballbox}Fix a point $x_0 \in \M$ and a system of privileged coordinates $z_1,\dots,z_n$ at $x_0$. Then there exist positive constants $a(x_0), A(x_0) > 0,$ and $\sigma(x_0) > 0$ such that for all $x$ with $d(x_0, x) < \sigma(x_0)$,
\begin{equation}\label{eqn:bbeqn}
a(x_0) \|z(x)\|_{x_0} \leq d(x_0, x) \leq A(x_0) \|z(x)\|_{x_0}.
\end{equation}
\end{theorem}
Constructing the coordinate system $z$ thus gives structure to how sub-Riemannian distance behaves locally; this structure may be used for steering, e.g. \cite{FJ-ea:05}.
It can be shown that there exists a continuously varying system of privileged coordinates on $\M$ so that the inequality \eqref{eqn:bbeqn} holds at all $x_0$ for continuous positive functions $a(\cdot), A(\cdot),$ and $\sigma(\cdot)$ on $\M$.  Let us assume that the system $\Sigma$ is sufficiently regular such that there exist bounds $0 < \amin \leq a(x) \leq A(x) \leq \Amax < \infty$ and $\sigma(x) \geq \sigmamin > 0$ for all $x \in \M$. We state a pair of lemmas (whose proofs are provided in the Appendix) concerning how privileged coordinates relate to the Euclidean notions of volume and distance.

\begin{lemma}[Box volume]\label{lem:boxvol} Fix $x \in \M$. The volume of $\boxw(x, r)$ is given by $\mu(\boxw(x, r)) = r^D$ where $D = \sum_{i=1}^n w_i$.
\end{lemma}

\begin{lemma}[Distance comparison]\label{lem:distancecomp} Fix a point $x_0 \in \M$ and a system of privileged coordinates $z_1,\dots,z_n$ at $x_0$. Then there exists a positive constant $\theta(x_0) > 0$ such that for all $x$ with $\|x_0 - x\| \leq \theta(x_0)$,
\[
\|x_0 - x\| \leq d(x_0, x) \leq 2\Amax\|x_0 - x\|^{1/\rank}.
\]
\end{lemma}

\section{Problem Formulation}\label{sec:probformulation}

Let $\M \subset \R^n$ be the manifold defining a configuration space of a robotic system. Let $\Mobs \subset \M$ be the obstacle region, such that $\M \setminus \Mobs$ is an open set, and denote the obstacle-free space as $\Mfree = \cl(\M \setminus\Mobs)$. The starting configuration $\xinit$ is an element of $\Mfree$, and the goal region $\Mgoal$ is an open subset of $\Mfree$.  The  trajectory planning problem is denoted by the tuple $(\Sigma, \Mfree, \xinit, \Mgoal)$, where $\Sigma$, as discussed in Section \ref{sec:background}, denotes a driftless control-affine system. A dynamically feasible trajectory $x$ is \emph{collision-free} if $x(t) \in \Mfree$ for all $t\in [0,T]$. A trajectory $x$ is said to be \emph{feasible} for the trajectory planning problem $(\Sigma, \Mfree, \xinit, \Mgoal)$ if it is dynamically feasible, collision-free, $x(0) = \xinit$, and $x(T) \in \cl(\Mgoal)$.

Let $\X$ be the set of all feasible paths. A \emph{cost function} for $(\Sigma, \Mfree, \xinit, \Mgoal)$ is a function $c:\X \to \reals_{\geq 0}$ from the set of paths to the nonnegative real numbers; in this paper we consider the cost function $c(x) = \ell(x)$ defined as the arc length of $x$ with respect to the Euclidean metric in $\M$. The objective is to find the feasible path with minimum associated cost; this minimum is achieved as long as $U$ is closed and bounded \cite{Lavalle:06}. The optimal trajectory planning problem is then defined as follows:

\begin{quote}{\bf Optimal motion planning for driftless systems}: 
Given a  trajectory planning problem $(\Sigma, \Mfree, \xinit, \Mgoal)$ and an arc length function $c:~\X \to \reals_{\geq 0}$, find a feasible path $x^{*}$ such that $c(x^{*} )= \min\{c(x):x \text{ is feasible}\}$. If no such path exists, report failure.
\end{quote}
Our analysis will rely on two key sets of assumptions, relating, respectively, to the underlying system $\Sigma$ and the problem-specific parameters $\Mfree, \xinit, \Mgoal$.

\subsubsection{Assumptions on system}\label{sec:assum_sys}
As in Section~\ref{sec:background}, we require from $\Sigma$ that a) the vector fields $g_1,\dots,g_m$ are bracket generating, b) the configuration space $\M$ contains only regular points, c) there exist constants $\amin, \Amax, \sigmamin$ such that Theorem~\ref{thm:ballbox} holds with these values at all $x_0 \in \M$, and d) there exists a bounding constant $\thmin$ such that $0 < \thmin < \theta(x_0)$ for all $x_0\in\M$, where $\theta(x_0)$ is as defined in Lemma~\ref{lem:distancecomp}. Assumption (a) is a basic requirement for the system to be controllable, let alone optimally controlled, while (b), (c), and (d) ensure that there are no extreme regions of the configuration space  which would require an unbounded sample density to capture their geometry. These assumptions will be collectively referred to as $A_\Sigma$.

\subsubsection{Assumptions on problem parameters}
We require that the goal region $\Mgoal$ has \emph{regular boundary}, that is there exists $ \xi > 0$ such that $\forall y \in \partial \Mgoal$, there exists $ z \in \Mgoal$ with $B^e(z, \xi) \subseteq \Mgoal$ and $y \in \partial B^e(z, \xi)$. This requirement that the boundary of the goal region has bounded curvature ensures that a point sampling procedure may expect to select points in the goal region near any point on its boundary. 

We also make requirements on the \emph{clearance} of a trajectory, i.e.,  its ``distance" from $\Mobs$, standard for sampling-based methods \citep{Karaman.Frazzoli:IJRR2011}. For a given $\delta>0$, the $\delta$-interior of $\Mfree$ is defined as the set of all states that are at least a Euclidean distance $\delta$ away from any point in $\Mobs$. A collision-free path $x$ is said to have strong $\delta$-clearance if it lies entirely inside the $\delta$-interior of $\Mfree$.  A collision-free path $x$ is said to have weak $\delta$-clearance if there exists a path $x'$ that has strong $\delta$-clearance and there exists a homotopy $\psi$, with $\psi(0)= x$ and $\psi(1) = x^{\prime}$ that satisfies the following three properties: (a) $\psi(\alpha)$ is a dynamically feasible path for all $ \alpha \in (0, 1]$, (b) $\lim_{\alpha\rightarrow0} c(\psi(\alpha)) = c(x)$, and (c) for all $ \alpha \in (0, 1]$, $\psi(\alpha)$ has strong $\delta_{\alpha}$-clearance for some $\delta_{\alpha}>0$. Properties (a) and (b) are required since pathological obstacle sets may be constructed that squeeze all optimum-approximating homotopies into undesirable motion. In practice, however, as long as $\Mfree$ does not contain any passages of infinitesimal width, the fact that $\Sigma$ is bracket-generating will allow every trajectory to be weak $\delta$-clear.

\section{Probabilistic Exhaustivity of Sampling Schemes under Driftless Constraints}\label{sec:exhaust}

In this section we prove a key result characterizing random sampling schemes for motion planning under driftless differential constraints: any feasible trajectory through the configuration space $\M$ is  ``traced'' arbitrarily well by connecting randomly distributed points from a sufficiently large sample set covering the configuration space. We will refer to this property as \emph{probabilistic exhaustivity}. Note that this notion is much stronger than the usual notion of probabilistic completeness in motion planning, where the requirement is that \emph{at least one} feasible trajectory is traced. The notion of probabilistic exhaustivity, besides being a result of independent interest, is a strong tool in proving asymptotic optimality of sampling-based motion planning algorithms, as will be shown in Section \ref{sec:algo} (specifically, we focus on differential variants of \PRM \citep{Karaman.Frazzoli:IJRR2011} and \FMT  \cite{LJ-ES-AC-ea:15}).

Let $x:[0,T]\rightarrow \M$ satisfy the system~\eqref{eqn:dcasys}. Given a set of waypoints $\{y_m\}_{m=1}^M \subset \M$, we associate a dynamically feasible trajectory $y^*:[0,S]\rightarrow \M$ that connects the nodes $y_1,\dots,y_M$ in order so that each connection is locally optimal, i.e. each path segment connecting $y_m$ to $y_{m+1}$ has length $d(y_m,y_{m+1})$. We consider the waypoints $\{y_m\}$ to \emph{$(\eps, r)$-trace} the trajectory $x$ if: a) $d(y_m,y_{m+1}) \leq r$ for all $m$, b) the cost of $y^*$ is bounded as $c(y^*) \leq (1+\eps)c(x)$, and c) the distance from any point of $y^*$ to $x$ is no more than $r$, i.e. $\min_{t\in[0,T]} d(y(s), x(t)) \leq r$ for all $s \in [0,S]$. In the context of sampling-based motion planning, we may expect to find closely tracing $\{y_m\}$ as a subset of the sampled points, provided the sample size is large. We formalize this notion in Theorem~\ref{thm:pathtracing}, the proof of which requires three technical lemmas (the proofs of these lemmas are provided in the Appendix).

\begin{lemma}\label{lem:ptubound} Let $x$ be a dynamically feasible trajectory and consider a partition of the time interval $0 = \tau_1 < \tau_2 < \dots < \tau_M = T$. Suppose that $\{y_m\} \subset \M$ satisfy a) $y_m \in B(x(\tau_m), \rho)$ for all $m \in \{1,\dots,M\}$, and b) more than a $(1-\alpha)$ fraction of the $y_m$ satisfy $y_m \in B(x(\tau_m), \beta \rho)$ for a parameter $\beta \in (0,1)$. Then the cost $c(y^*)$ of the trajectory $y^*$ sequentially connecting the nodes $y_1,\dots,y_M$ is upper bounded as
\[
c(y^*) \leq c(x) + 2M\rho(\beta + \alpha - \alpha\beta).
\]
\end{lemma}

\begin{remark}\label{lem:ptuboundpp} If we further assume that $y_1 = x(0)$, then the bound
$
c(y^*) \leq c(x) + 2(M-1)\rho(\beta + \alpha - \alpha\beta)
$
holds.
\end{remark}

Let $\SF(n)$ denote a set of $n$ points sampled independently and identically from the uniform distribution on $\Mfree$.

\begin{lemma}\label{lem:samplesmall} Fix $n \in \N$, $\alpha \in (0,1)$, and let $S_1, \dots, S_M$ be disjoint subsets of $\Mfree$ with
\[
\mu(S_m) = \mu(S_1) \geq \left(\frac{2+\log(1/\alpha)}{n}\right)e^2 \mu(\Mfree),
\]
for each $m$. Let $V = \SF(n)$ and define
\[
K_n := \#\{m \in \{1,\dots,M\} : S_m \cap V = \emptyset\}.
\]
Then $\p{K_n \geq \alpha M} \leq \frac{e^{-\alpha  M}}{1-e^{-n}}$.
\end{lemma}

\begin{lemma}\label{lem:samplebig} Fix $n \in \N$ and let $T_1, \dots, T_M$ be subsets of $\Mfree$, possibly overlapping, with
\[
\mu(T_m) = \mu(T_1) \geq \kappa\left(\log n/n \right)\mu(\Mfree)
\]
for each $m$ and some constant $\kappa > 0$. Let $V = \SF(n)$ and denote by $E_m$ the event that $T_m \cap V = \emptyset$ for each $m$. Then
\[
\p{\bigvee_{m=1}^M E_m} \leq Mn^{-\kappa}.
\]
\end{lemma}

Before stating the theorem and proof in full, we sketch our approach for proving probabilistic exhaustivity. Given a path to be traced with waypoints from a sample set, we tile the span of the path with two sequences of concentric sub-Riemannian balls -- a sequence of ``small'' balls and a sequence of ``large'' balls.  With high probability, all but a tiny $\alpha$ fraction of the small balls will contain a point from the sample set (Lemma~\ref{lem:samplesmall}), and for any small balls that don't contain such a point we ensure that the concentric large ball does (Lemma~\ref{lem:samplebig}). We take these points as a sequence of waypoints which tightly follows the reference path with few exceptions, and never has a gap over any section of the reference path when it deviates. We then use the metric inequality for $d$ to bound the total cost of the waypoint trajectory (Lemma~\ref{lem:ptubound}).
\begin{theorem}[Probabilistic exhaustivity]\label{thm:pathtracing}
Let $\Sigma$ be a DCA system satisfying the assumptions $A_\Sigma$ and suppose $x: [0,T] \rightarrow \Mfree$ is a dynamically feasible trajectory with strong $\delta$-clearance, $\delta > 0$. Let $V = \{\xinit\} \cup \SF(n)$, $\eps > 0$, and for fixed $n$ consider the event $A_n$ that there exist $\{y_m\}_{m=1}^M \subset V$ which $(\eps, r_n)$-trace $x$, where
\[
r_n = 4 \Amax (1 + \eta)^{1/D}\left(\frac{\mu(\Mfree)}{D}\right)^{1/D} \left(\frac{\log n}{n}\right)^{1/D}
\]
for a parameter $\eta \geq 0$. Then, as $\nti$, the probability that $A_n$ does not occur is asymptotically bounded as $\p{A^c_n} = O(n^{-\eta/D} \log^{-1/D} n)$.
\end{theorem}
\begin{proof}
Note that in the case $c(x) = 0$ we may pick $y_1 = x(0)$ to be the only waypoint and the result is trivial. Therefore assume $c(x) > 0$. Fix $n$ sufficiently large so that $r_n/2 \leq \sigmamin$. Take $x(\tau_{n,m})$ to be points spaced along $x$ at sub-Riemannian distances $r_n/2$; more precisely let $\tau_1 = 0$, and for $m=2,3,\dots$ consider
\[
\tau_{n,m} \!=\! \min\left(\left\{\tau \in (\tau_{n,m-1}, 1): d(x(\tau), x(\tau_{n,m-1})) \geq r_n/2\right\}\right).
\]
Let $M_n$ be the first $m$ for which the set is empty; take $\tau_{n,M_n} = T$. Note that since the distance $d$ is an infimum over all feasible trajectories, one of which is the segment of $x$ between $\tau_{n,m}$ and $\tau_{n,m+1}$, we have the bound $M_n \leq \lceil 2 c(x)/r_n \rceil$.

We now make the identification $\alpha = \beta = \eps/2$ anticipating the application of Lemma~\ref{lem:ptubound}. Take $\rho_n = r_n/4$ and define a sequence of sub-Riemannian balls $B_{n,1},\dots,B_{n,M_n}$ centered along the trajectory $x$ by $B_{n,m} = B(x_m, \rho_n)$, where $x_m = x(\tau_{n,m})$ for $m \in \{1,\dots,M_n\}$. Within these balls define a concentric sequence of smaller, disjoint balls $B^\beta_{n,m} = B(x_m, \beta \rho_n)$ for each $m \in \{1,\dots,M_n\}$. Note that since $\boxw(x_m, \rho_n/\Amax) \subset B_{n,m}$, we have the volume lower bound
$
\mu(B_{n,m}) \geq \left(\rho_n/\Amax\right)^D
$
and similarly
$
\mu(B^\beta_{n,m}) \geq \left(\beta \rho_n/\Amax\right)^D.
$
Denote $K^\beta_n := \#\{m \in \{1,\dots,M_n\} : B^\beta_{n,m} \cap V = \emptyset\}$. We consider the event $\tilde A_n$ that every large ball $B_{n,m}$, as well as at least a $(1-\alpha)$ fraction of the small balls $B^\beta_{n,m}$, contains at least one point of $V$:
\[
\tilde A_n =  \left\{K^\beta_n < \alpha M_n\right\} \,\wedge\, \bigwedge_{m=1}^{M_n} \{B_{n,m} \cap V \neq \emptyset\}.
\]
We claim that $\tilde A_n$ implies the event $A_n$ that there exist $(\eps, r_n)$-tracing $\{y_m\} \subset V$. If $\tilde A_n$ holds, then we select waypoints $\{y_m\}_{m=1}^{M_n} \subset V$ such that $y_m \in B_{n,m}$ for every point, as well as $y_m \in B^\beta_{n,m}$ for at least a $(1-\alpha)$ fraction of the points. In particular let us select $y_1 = x(0)$.

First note that $y_m \in B_{n,m}$ for each $m$ implies
\begin{align*}
d(y_m,y_{m+1}) &\leq d(y_m, x_m) + d(x_m, x_{m+1}) \\
               & \qquad + d(x_{m+1}, y_{m+1}) \\
               &\leq r_n/4 + r_n/2 + r_n/4 = r_n.
\end{align*}
Next, applying Remark~\ref{lem:ptuboundpp} we have
\begin{align*}
c(y^*) &\leq c(x) + 2(M_n-1) \rho_n(\beta + \alpha - \alpha\beta) \\
       &\leq c(x) + 2 (2 c(x)/r_n) (r_n/4) \eps\\
       &\leq (1+\eps) c(x).
\end{align*}
Finally we check that any point $y^*(t)$ is within distance $r_n$ from a point on $x$. To see this suppose that $y^*(t)$ lies on the shortest path connecting $y_m$ to $y_{m+1}$. Note that
\[
d(y_m, y^*(t)) + d(y^*(t), y_{m+1}) = d(y_m, y_{m+1}) \leq r_n.
\]
Then we may write
\begin{align*}
d(x_m, y^*(t)) + d(x_{m+1}, y^*(t)) &\leq d(x_m, y_m) \\
                                    & \quad\ + d(x_{m+1}, y_{m+1}) + r_n\\
                                    &\leq (3/2)r_n,
\end{align*}
so that $\min\{d(x_m, y^*(t)), d(x_{m+1}, y^*(t))\} \leq (3/4)r_n < r_n$. Thus the waypoints $\{y_m\}$ $(\eps, r_n)$-trace $x$.

Now making the identifications $S_m = B^\beta_{n,m}$ and $T_m = B_{n,m}$ in Lemmas~\ref{lem:samplesmall} and \ref{lem:samplebig} respectively, we see that by considering $n \geq N_1$ sufficiently large so that $\log N_1 \geq D\beta^{-D}(2+\log(1/\alpha))e^2$ (satisfying the volume assumption of Lemma~\ref{lem:samplesmall}), we compute the union bound
\begin{align*}
\p{A^c_n} &\leq \p{\tilde A^c_n} \\
          &\leq \p{K^\beta_n \geq \alpha M_n} + \p{\bigvee_{m=1}^{M_n} \{B_{n,m} \cap V = \emptyset\}} \\
          &\leq \frac{e^{-\alpha  M_n}}{1-e^{-n}} + M_n n^{-(1+\eta)/D}.
\end{align*}
Now, $c(x) > 0$ and $r_n  = \Theta((\log n /n)^{1/D})$ together imply that $M_n = \Theta((n/\log n)^{1/D})$. The second term in the bound dominates as $\nti$, and $\p{A^c_n} = O(n^{-\eta/D} \log^{-1/D} n)$.
\end{proof}

\section{Optimal Sampling-based Algorithms for Driftless  Control-Affine Systems}\label{sec:algo}
In this section we present two algorithms for the motion planning problem with driftless control-affine systems. The first algorithm, named the Differential Probabilistic RoadMap algorithm (\DPRM), is a derivation of the \PRM algorithm presented in \citep{Karaman.Frazzoli:IJRR2011}, while the second algorithm, named the Differential Fast Marching Tree algorithm (\DFMT), is a derivation of the \FMT algorithm presented in \cite{LJ-ES-AC-ea:15}. This section provides a description of both algorithms, while the next section focuses on their theoretical characterization (chiefly, their asymptotic optimality property).

As in Section~\ref{sec:exhaust}, let $\texttt{SampleFree}(k)$ be a function that returns a set of $k \in \mathbb{N}$ states sampled independently and identically from the uniform distribution on $\Mfree$. These sampled states are connected as vertices in a graph from which a solution trajectory will be computed.
Given two vertices $x_1$ and $x_2$, we denote with the edge $(x_1,x_2)$ an optimal cost trajectory from $x_1$ to $x_2$ \emph{neglecting} obstacle constraints. Let $\texttt{CollisionFree}(x_1, x_2)$ denote the boolean function which returns true if and only if the edge $(x_1, x_2)$ does not intersect an obstacle. Given a set of vertices $V$, a state $x \in \M$, and a threshold $r > 0$, let $\texttt{Near}(V, x, r)$ be a function that returns the set of states $\{v \in V : \|v\|_x < r/\amin\}$. Given a graph $G = (V, E)$, where $V$ is the vertex set and $E$ is the edge set, and a vertex $x \in V$, let $\texttt{Cost}(x, G)$ be the function that returns the cost of the shortest path in the graph $G$ between the vertices $\xinit$ and $x$. Let $\texttt{Path}(x, G)$ be the function that returns the path achieving that cost.

The \DPRM algorithm is given in Algorithm \ref{prmalg}, while the \DFMT algorithm is given in Algorithm \ref{fmtalg}. \DPRM works by sampling a set of points within $\Mfree$ and connecting each state to every state in a local neighborhood around it, provided that the connection is collision free. The resulting graph spans $\Mfree$, with the local connections combining to yield a global ``roadmap'' for traveling between any two states, not just $\xinit$ and the goal. The least cost path in the graph between $\xinit$ and $\Mgoal$, computed, e.g., using the Dijkstra's algorithm, is output by \DPRM.

The \DFMT algorithm essentially implements a streamlined version of \DPRM by performing a ``lazy'' dynamic programming recursion, this time during the state connection phase instead of as a last step, to grow a \emph{tree} of trajectories which moves steadily outward in cost-to-come space. In the Algorithm~\ref{fmtalg} outline, the set $W$ consists of all of the nodes that have not yet been added into the tree, while $H$ is comprised only of nodes that are in the tree. In particular, while $H$ keeps track of nodes which have already been added to the tree, nodes are removed from $H$ if they are not near enough to the edge of the expanding tree to actually have any new connections made with $W$ (see \cite{LJ-ES-AC-ea:15} for further details). At each iteration, \DFMT examines the neighborhood of a state in $H$ and \emph{only} considers locally-optimal (assuming no obstacles) connections for potential inclusion in the tree (see line \eqref{line:locon}).
By only checking for collision on the locally-optimal (assuming no obstacles) connection, as opposed to every possible connection (essentially what is done in \DPRM), \DFMT saves a large (indeed unbounded as the number of vertices increases) number of collision-check computations.
A more detailed explanation of the meaning of the dual sets $H$ and $W$ and of the philosophy behind \DFMT can be found in \cite{LJ-ES-AC-ea:15}.

The main differences of these algorithms with respect to their geometric counterparts (i.e., \PRM and \FMT) are that (i) near vertices lie within the privileged coordinate box $\boxw(x, r/\amin)$ rather than within the Euclidean ball, and (ii) the edges connecting vertices are optimal \emph{trajectories} rather than straight lines. Indeed, attempting connections within the sub-Riemannian ball $B(x,r)$ would be the best analogy to the geometric planning case, and is preferable for efficiency if possible. The structure of this set is difficult to compute exactly, however, which motivates the choice $\boxw(x, r/\amin)$ as a tractable approximation (as also suggested in \cite{SK-EF:13}). Checking connections within this superset of $B(x,r)$ induces an extra run time factor of at most $(\Amax/\amin)^D$, a bound on the volume ratio between the two sets.
\setlength{\textfloatsep}{6pt}
\begin{algorithm}[t]
\caption{Differential Probabilistic RoadMap (\DPRM)}
\label{prmalg}
\algsetup{linenodelimiter=}
\begin{algorithmic}[1]
\STATE $V \leftarrow \{\xinit\} \cup \texttt{SampleFree}(n)$; $E \leftarrow \emptyset$
\FOR{\textbf{each} $v\in V$}
\FOR{\textbf{each} $u\in \texttt{Near}(V \backslash \{v\}, v, r_n)$}
\IF{$\texttt{CollisionFree}(u,v)$} 
\STATE $E\leftarrow E\cup \{\{u,v\}\}$
\ENDIF
\ENDFOR
\ENDFOR
\STATE $V^*\leftarrow \argmin_{v\in V\cap\xgoal} \texttt{Cost}(v, G = (V, E))$
\IF{$V^* \neq \emptyset$} 
\STATE $v^*\leftarrow $ random vertex in $V^*$
\RETURN $\texttt{Path}(v^*, G = (V, E))$
\ELSE
\RETURN Failure
\ENDIF
\end{algorithmic}
\end{algorithm}

\begin{algorithm}[t]
\caption{Differential Fast Marching Tree (\DFMT)}
\label{fmtalg}
\algsetup{linenodelimiter=}
\begin{algorithmic}[1]
\STATE $V \leftarrow \{\xinit\} \cup \texttt{SampleFree}(n)$; $E \leftarrow \emptyset$
\STATE $W \leftarrow V \backslash \{\xinit\}$; $H \leftarrow \{\xinit\}$
\STATE $z \leftarrow \xinit$
\WHILE{$z \notin \Mgoal$}
\STATE $H_{\text{new}} \leftarrow \emptyset$
\STATE $X_{\text{near}} = \texttt{Near}(V \backslash \{z\}, z, r_n) \cap W$
\FOR{$x \in X_{\text{near}}$}
\STATE $Y_{\text{near}} \leftarrow \texttt{Near}(V \backslash \{x\}, x, r_n) \cap H$ \label{line:intersect}
\STATE $y_{\text{min}} \leftarrow \arg\min_{y \in Y_{\text{near}}}\{\texttt{Cost}(y, T = (V, E)) \!+\! d(y,x)\}$
\IF{$\texttt{CollisionFree}(y_{\text{min}}, x)$} \label{line:locon}
\STATE $E \leftarrow E \cup \{\{y_{\text{min}}, x\}\}$
\STATE $H_{\text{new}} \leftarrow H_{\text{new}} \cup \{x\}$ \label{alg:H_1}
\STATE $W \leftarrow W \backslash \{x\}$
\ENDIF
\ENDFOR
\STATE $H \leftarrow (H \cup H_{\text{new}}) \backslash \{z\}$ \label{alg:H_2}
\IF{$H = \emptyset$}
\RETURN Failure
\ENDIF
\STATE $z \leftarrow \arg\min_{y \in H}\{\texttt{Cost}(y, T = (V, E))\}$
\ENDWHILE
\RETURN $\texttt{Path}(z, T = (V, E))$
\end{algorithmic}
\end{algorithm}

\section{Asymptotic Optimality of \DFMT and \DPRM}\label{sec:AO}
In this section, we prove the asymptotic optimality of \DFMT and \DPRM (obtained as a simple corollary). We conclude the section by providing a discussion of the results.

\subsection{Asymptotic Optimality of \DFMT and \DPRM}
The following theorem presents a result comparing the output cost of \DFMT to the cost of any feasible trajectory. 

\begin{theorem}[\DFMT cost comparison]\label{thm:DFMTccomp} Let $(\Sigma, \Mfree, \xinit, \Mgoal)$ be a  trajectory planning problem satisfying the assumptions $A_\Sigma$ and suppose $x: [0,T] \rightarrow \M$ is a feasible path with strong $\delta$-clearance, $\delta > 0$. Assume further that $x$ extends into the interior of $\Mgoal$, i.e. there exists $\gamma>0$ such that $B(x(T), \gamma) \subset \Mgoal$.
Let $c_n$ denote the cost of the path returned by \DFMT with $n$ vertices using a radius
\[
r_n = 4 \Amax (1 + \eta)^{1/D} \left(\frac{\mu(\Mfree)}{D}\right)^{1/D} \left(\frac{\log n}{n}\right)^{1/D}
\]
for a parameter $\eta \geq 0$. Then for fixed $\eps > 0$
\[
\p{c_n > (1+\eps)c(x)} = O(n^{-\eta/D}\log^{-1/D} n).
\]
\end{theorem}
\begin{proof}
The proof of this theorem is conceptually similar to the one of Theorem 1 in \cite{LJ-ES-AC-ea:15}. The details are provided in the Appendix.
\end{proof}

We are now in a position to state the optimality result for \DFMT (note that the result also provides a \emph{convergence rate} bound). The optimality of \DPRM will follow as a corollary.

\begin{theorem}[\DFMT asymptotic optimality]\label{thm:DFMTAO} Let $(\Sigma, \Mfree, \xinit, \Mgoal)$ be a  trajectory planning problem, satisfying the assumptions $A_\Sigma$ and with $\Mgoal$ $\xi$-regular, such that there exists an optimal path $x^*: [0,T^*]\rightarrow\M$ with weak $\delta$-clearance for some $\delta > 0$. Let $c^*$ denote the arc length of $x^*$, and let $c_n$ denote the cost of the path returned by \DFMT with $n$ vertices using the radius
\[
r_n = 4 \Amax (1 + \eta)^{1/D} \left(\frac{\mu(\Mfree)}{D}\right)^{1/D} \left(\frac{\log n}{n}\right)^{1/D}
\]
for a parameter $\eta \geq 0$. Then for fixed $\eps > 0$
\[
\p{c_n > (1+\eps)c^*} = O(n^{-\eta/D}\log^{-1/D} n).
\]
\end{theorem}
\begin{proof}
Note that $c^* = 0$ implies $x_{\text{init}} \in cl(\Mgoal)$, and \DFMT will terminate immediately and optimally. In this case the result is trivial, therefore assume $c^* > 0$. Let $\psi$ denote the homotopy given by the weak $\delta$-clearance definition for $x^*$. Using the fact that $\lim_{\alpha\rightarrow0} c(\psi(\alpha)) = c^*$, we pick an $\alpha_0 > 0$ such that $c(\psi(\alpha_0)) \leq (1 + \eps/4)c^*$. Denote $x:=\psi(\alpha_0)$. We now extend $x$ into the interior of $\Mgoal$ so that we may apply Theorem~\ref{thm:DFMTccomp}. Since $x^*$ is optimal, $x(T) = x^*(T^*) \in \partial \Mgoal$ and the $\xi$-regularity of $\Mgoal$ means there exists  $z \in \Mgoal$ with $B^e(z, \xi) \subseteq \Mgoal$ and $y \in \partial B^e(x(T), \xi)$. Pick
$
\rho = \min\left\{\xi, \frac{\delta_{\alpha_0}}{2}, \frac{1}{2 \Amax}\left(\frac{\eps c^*}{4}\right)^{\rank}, \thmin\right\},
$
and take $z'$ on the straight line segment between $x(T)$ and $z$ with $\|x(T) - z'\| \leq \rho$. Then $B(z', \rho) \subset B^e(z', \rho) \subset B^e(z, \xi) \subset \Mgoal$. The last two terms in the minimum above ensure, from Lemma~\ref{lem:distancecomp}, that $d(x(T), z') \leq (\eps/4)c^*$. Consider the extension $x'$ of $x$ constructed by concatenating $x$ with the shortest sub-Riemannian path between $x(T)$ and $z'$. The cost of $x'$ is bounded as
\[
c(x') \leq c(x') + (\eps/4)c^* \leq (1 + \eps/2)c^*.
\]
The strong $(\delta_{\alpha_0}/2)$-clearance of $x'$ is established by noting for any $p$ along the path between $x(T)$ and $z'$:
\begin{align*}
\inf_{a \in \Mobs} \|p - a\| &\geq \inf_{a \in \Mobs} \|x(T) - a\| - \|p - x(T)\| \\&\geq \delta_{\alpha_0} - \delta_{\alpha_0}/2 = \delta_{\alpha_0}/2.
\end{align*}
Then by Theorem~\ref{thm:DFMTccomp}, with the approximation factor $\eps/4$, we have for $\eps \in (0,1)$ that
\begin{align*}
\p{c_n > (1+\eps)c^*} &\geq \p{c_n > (1+\eps/4)(1+\eps/2)c^*} \\&= O(n^{-\eta/D}\log^{-1/D} n).
\end{align*}
For $\eps \geq 1$ the desired result follows from the monotonicity in $\eps$ of the above probability.
\end{proof}
\begin{remark}
There is an implicit dependence on the problem parameters and approximation factor $\eps$ in the convergence rate bound given above in Theorem~\ref{thm:DFMTAO}; these fixed parameters influence the threshold $n > N$ after which the stated asymptotic bound holds.  
\end{remark}

We conclude this section with the asymptotic optimality result for \DPRM.
\begin{corollary}[\DPRM asymptotic optimality]\label{thm:DPRMAO} Let $(\Sigma, \Mfree, \xinit, \Mgoal)$ be a  trajectory planning problem, satisfying the assumptions $A_\Sigma$ and with $\Mgoal$ $\xi$-regular, such that there exists an optimal path $x^*: [0,T^*]\rightarrow\M$ with weak $\delta$-clearance for some $\delta > 0$. Let $c^*$ denote the arc length of $x^*$, and let $c_n$ denote the cost of the path returned by \DPRM with $n$ vertices using the radius
\[
r_n = 4 \Amax (1 + \eta)^{1/D} \left(\frac{\mu(\Mfree)}{D}\right)^{1/D} \left(\frac{\log n}{n}\right)^{1/D}
\]
for a parameter $\eta \geq 0$. Then for fixed $\eps > 0$
\[
\p{c_n > (1+\eps)c^*} = O(n^{-\eta/D}\log^{-1/D} n).
\]
\end{corollary}
\begin{proof}
Given the same sample set $V$, the trajectory tree constructed by \DFMT is a subgraph of the roadmap graph constructed by \DPRM. Hence the cost of the path returned by \DPRM is upper bounded by that of \DFMT. The desired result then follows immediately from Theorem~\ref{thm:DFMTAO}.
\end{proof}

\subsection{Discussion}\label{sec:disc}
We note that there is room for improvement in the choice of constants for the theorems and lemmas above (particularly in the connection radius $r_n$), which for the sake of clarity were not pushed as tight as possible. As a consequence, the asymptotic rate bound $O(n^{-\eta/D}\log^{-1/D} n)$ may be achieved with a connection radius smaller by a constant factor.
The convergence rate bounds for \DFMT and \DPRM, stated in terms of the sample size, are also tunable to specific implementations and planning problems. A greater local connection radius achieves a faster asymptotic convergence rate, but such a choice may ultimately result in greater algorithm execution time as more connections are checked.

The essential property of the system $\Sigma$ underpinning the asymptotic optimality results above is the ball-box Theorem~\ref{thm:ballbox}. Indeed, although the results we present nominally apply only to arc length, the proofs hold for any system with a cost metric $d$ on $\M$ satisfying the ball-box inequality~\eqref{eqn:bbeqn}. For example, if the ball-box inequality holds for arc length on a subset of dimensions of $\M$, \DFMT and \DPRM asymptotic optimality hold with respect to that subspace metric, a fact we make use of in our simulations. The constant $\amin$ ensures that the sub-Riemannian ball $B(x,r)$ is covered by the connection neighborhood of $x$ (i.e. the algorithms search far enough in each direction), while the constant $\Amax$ provides a lower bound on $\mu(B(x,r))$ that ensures the sample set will intersect the neighborhood with high probability (i.e. the algorithms examine enough volume along an optimal path). If checking state membership in balls $B(x,r)$ is computationally efficient,
 then algorithmic performance gains may be achieved by using $\texttt{Near}(V, x, r) = V \cap B(x,r)$, while maintaining the guarantees of Theorems~\ref{thm:DFMTccomp} and \ref{thm:DFMTAO}. Another consequence of this ball-box analysis foundation is the possibility of accommodating \emph{approximate} steering techniques (as in \cite{FJ-ea:05}) for computing local, obstacle-free state connections (one of the bottlenecks for \DPRM and \DFMT).
As long as \DFMT and \DPRM have access to a local planner that can connect two states exactly, but possibly with some approximation factor to the optimal distance, then the global optimality theorems hold up to that same factor.

\begin{figure}[!htbp]
 \centering
 \subfigure{\includegraphics[width=0.48\textwidth]{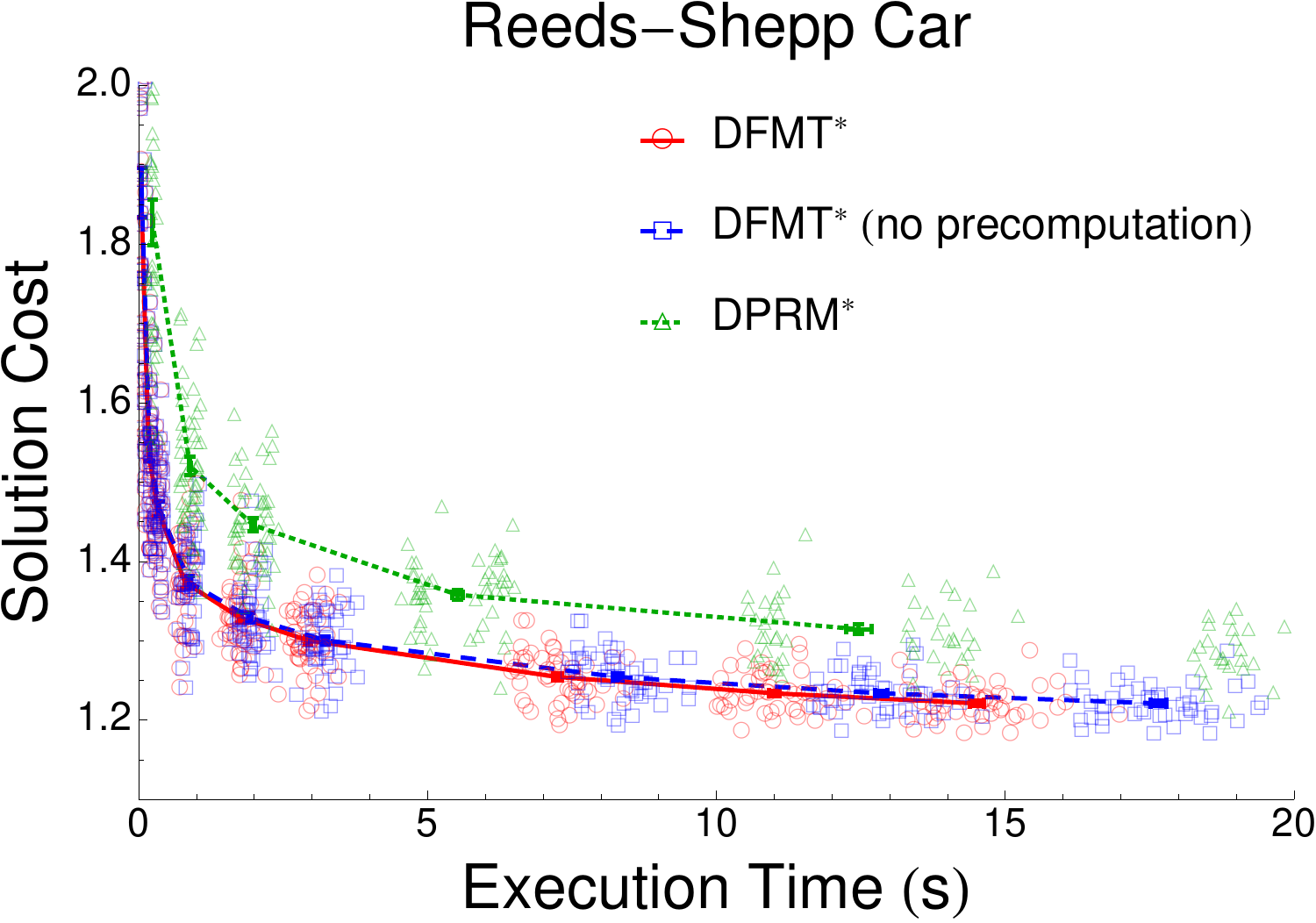}}
 \subfigure{\includegraphics[width=0.48\textwidth]{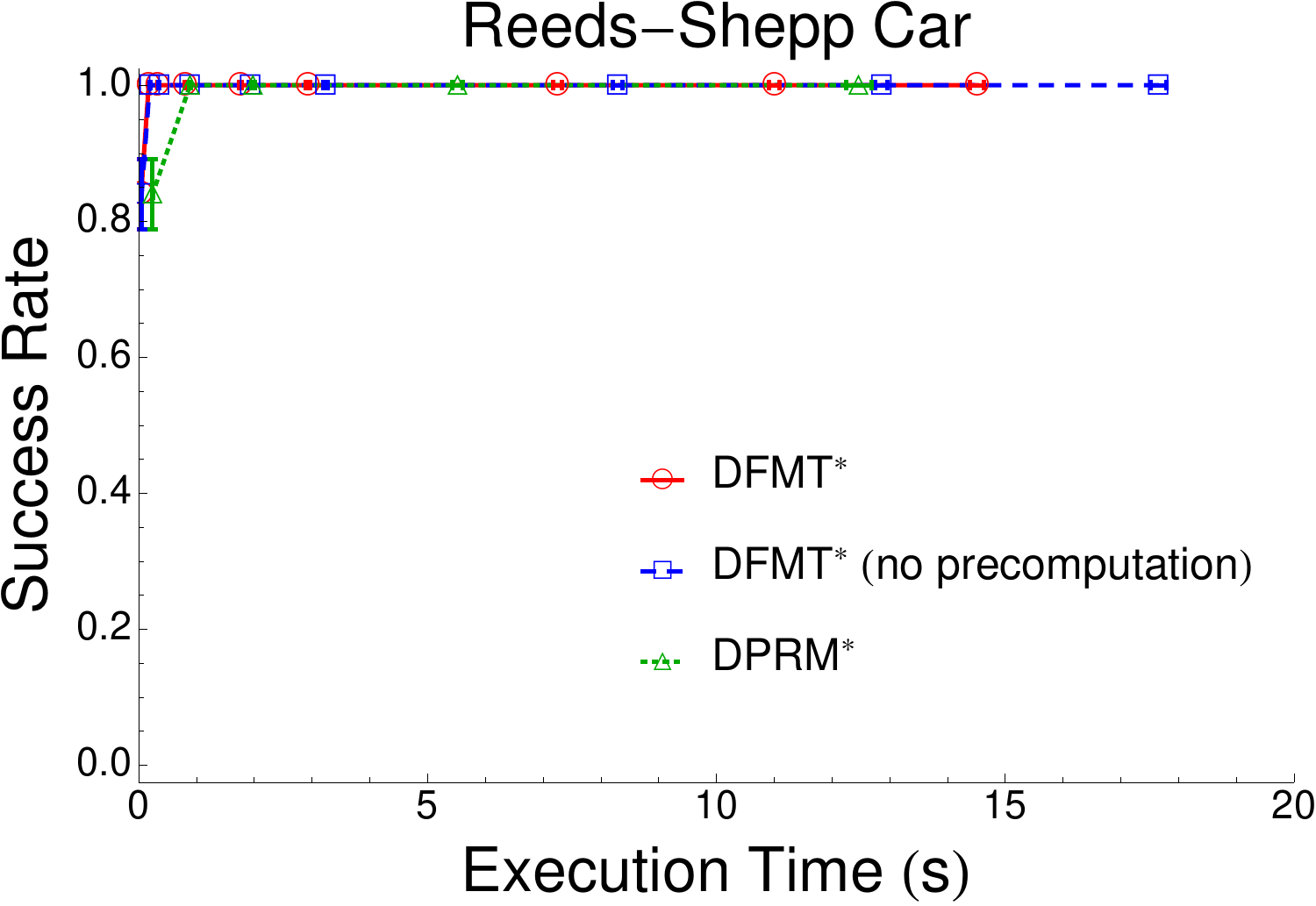}}
 \subfigure{\includegraphics[width=0.48\textwidth]{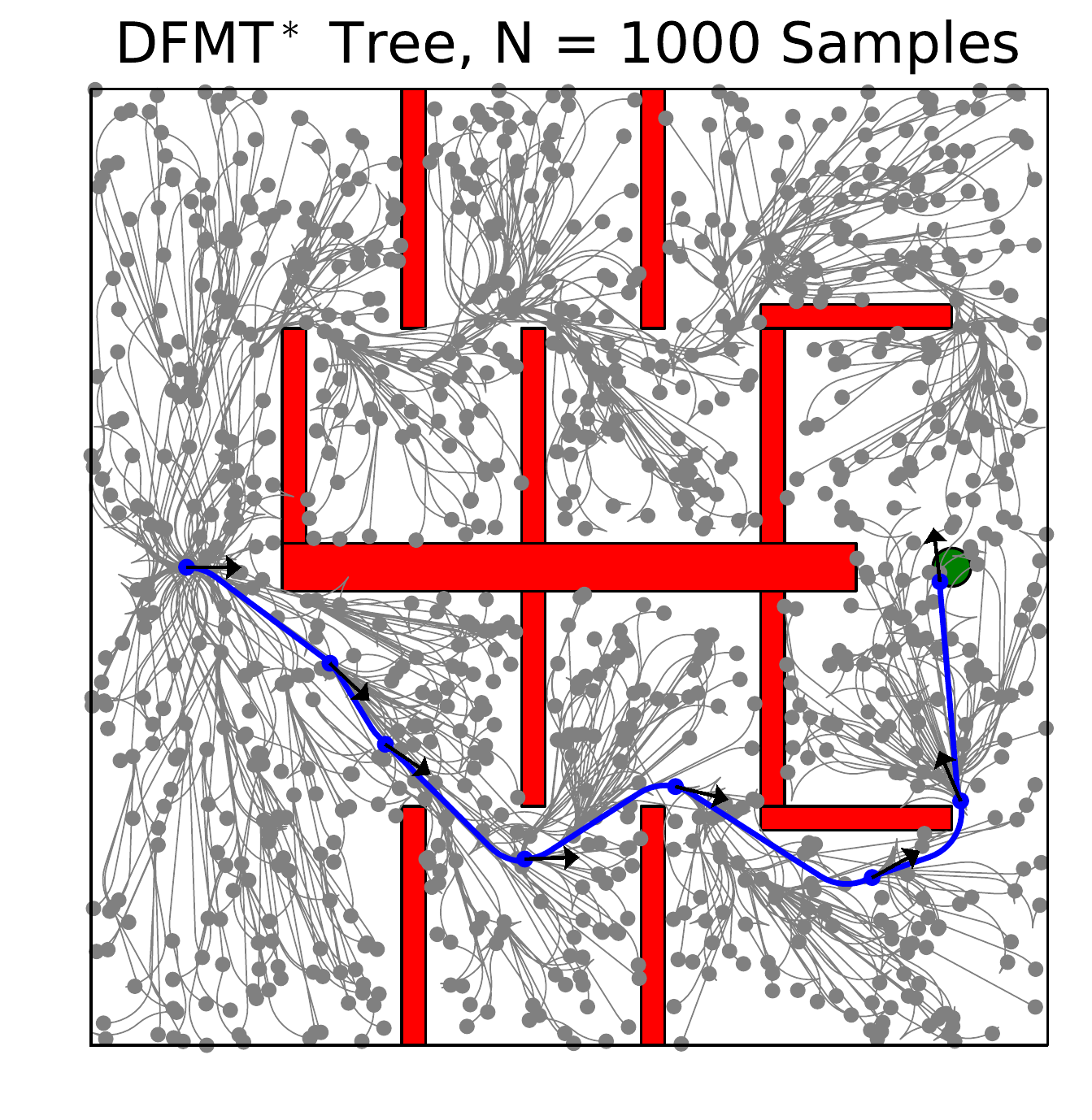}}
 \caption{Top: Simulation results for the Reeds-Shepp car system with a maze obstacle set. The error bars in each axis represent plus and minus one standard error of the mean for a fixed sample size $n$.
 Bottom: Example \DFMT tree for $n=1000$. The feasible path returned is highlighted in blue with car direction at each sample state denoted by an arrow.}
\label{fig:all}
\end{figure}

\section{Implementation and numerical experiments}\label{sec:sims}

The \DFMT and \DPRM algorithms were implemented in Julia and run using a Unix operating system with a 2.0 GHz processor and 8 GB of RAM.
Near neighbor sets were precomputed and cached at algorithm initialization after the sample set was selected. Note that for batch-processing (as opposed to ``anytime'') algorithms such as  \DFMT and \DPRM, one can precompute both near neighbor sets and sub-Riemannian distance ahead of time, as they do not depend on the obstacle configuration --- the price to pay is a moderate increase in memory requirements. 

We tested \DFMT and \DPRM on the Reeds-Shepp car system \cite{Reeds.Shepp:90}, a well-known example of a regular driftless control-affine system. Within the configuration manifold $\M=\R^2\times S^1$, trajectories $x(t) = (x^1(t), x^2(t), \theta(t))$ are subject to equation~\eqref{eqn:dcasys} with $g_1(x) = (\cos\theta, \sin\theta, 0)$ and $g_2(x) = (0, 0, 1)$. With the car constrained to move at unit speed with turning radius $R$, the control set is $U = \{-1, 1\}\times[-1/R,1/R]$. The privileged coordinate system at each point $x$ is defined by the vectors $g_1(x)$, $g_2(x)$, and $[g_1(x),g_2(x)] = (\sin\theta,-\cos\theta,0)$ with associated weight vector $(1,1,2)$. We aim to optimize planar arc length in the first two dimensions of $\M$, in which case the constants for the ball-box Theorem~\ref{thm:ballbox} are computed as $\amin = \sqrt{2R}$ and $\Amax = 2\sqrt{2R}$ for $\sigmamin = R$.

An analytic solution to the Reeds-Shepp optimal steering problem in the absence of obstacles is known \cite{Reeds.Shepp:90}, and consists of checking forty-six candidate curves for the one of minimum length. We precomputed and cached the optimal curve type between a state at the origin and any general state (using spatial symmetry this gives the optimal steering curve type between any two states in Reeds-Shepp space); the time for this operation was not included in our accounting of algorithm execution time. We used this data to implement the variants of \DFMT and \DPRM discussed in Section~\ref{sec:disc} where the exact sub-Riemannian ball is used for the $\texttt{Near}$ function instead of the privileged coordinate box. Collisions with obstacles were detected by approximating trajectories by piecewise linear interpolations, for which polytope intersection is simple to compute.

The simulation results are summarized in Figure~\ref{fig:all}. A maze was used for $\Mobs$, and our \DFMT and \DPRM implementations were run 50 times each on sample sizes up to $n=4000$ in the case of \DFMT and $n=1000$ in the case of \DPRM. We plot results for \DFMT run both with and without the precomputed neighbor sets and distance cache; the solutions returned are identical and differ only in their timing data. For the Reeds-Shepp system, local connection computation and collision checking takes the majority of the computation time, and the speed increase from the cache is modest. For a greater speedup at the cost of more memory usage, all local \emph{collision-free} connections could be cached ahead of time as well, leaving only collision checking with $\Mobs$ (the only detail specific to the problem instance) to be computed online.

\section{Conclusions}\label{sec:conc}

In this paper we presented a thorough theoretical framework to assess optimality guarantees of sampling-based algorithms for planning with driftless control-affine systems. We exploited this framework to design and analyze two novel sampling-based algorithms that are guaranteed to converge, as the number of samples increases, to an optimal solution. Our theoretical framework, which relies on the key notion of probabilistic exhaustivity, separates the probabilistic analysis of random sampling from the discrete analysis of combinatorial optimization algorithms, the two main elements in determining the performance of batch-processing motion planners (such as \DFMT and \DPRM). We hope that this analysis framework can prove useful to analyze other types of algorithms for this class of problems. We note that our analysis approach yields convergence rate guarantees, which shed additional light on the behavior of sampling-based algorithms for differentially-constrained motion planning problems.

This paper leaves numerous important extensions open for further research. One extension that appears quite tractable is generalizing the cost function from arc length to functions that may depend on state-space position and/or control effort. For example, the techniques established in this paper may be used, with a few modifications, to prove \DFMT and \DPRM variants for positionally weighted arc length costs, in which certain spatial regions incur higher or lower costs for traversal. A second extension of great interest  is the design and analysis of algorithms that address systems with \emph{drift} (of critical importance, for example, for spacecraft motion planning). The specific analysis techniques used in this paper rely heavily on the fact that driftless control-affine systems can be viewed as metric spaces, but asymmetric adaptations may be imagined within the same proof framework. Third, we plan to study bidirectional versions of \DFMT and \DPRM. Finally, determining how algorithm tuning parameters affect asymptotic convergence rates and, ultimately, run time merits more extensive theoretical and experimental analysis, especially for DCA systems beyond the basic proof-of-concept Reeds-Shepp example presented in this paper.

\renewcommand*{\bibfont}{\footnotesize}
\bibliographystyle{IEEEtran}
\bibliography{references,../../../bib/alias,../../../bib/main}

\section*{Appendix}
\begin{replemma}{lem:boxvol}[Box volume] Fix $x \in \M$. The volume of $\boxw(x, r)$ is given by $\mu(\boxw(x, r)) = r^D$ where $D = \sum_{i=1}^n w_i$.
\end{replemma}
\begin{proof}
From the formula \eqref{privCOC} we see that the Jacobian matrix $J_Z$ for the change of coordinates between the local frame $y_1,...,y_n$ and privileged coordinates $z_1,\dots, z_n$ is lower-triangular with all diagonal elements equal to $1$; thus the Jacobian determinant $|J_Z|$ is identically $1$. Since $Y_1(x),\dots,Y_n(x)$ are orthonormal, the change of coordinates from the Euclidean standard basis to the coordinates $y_1,\dots,y_n$, with Jacobian $J_Y$, satisfies $|J_Y|=1$. Then $\mu(\boxw(x, r)) = \int_{\|z\|_x \leq r} |J_Z^{-1}||J_Y^{-1}| dz = \int_{\|z\|_x \leq r} 1\, dz = r^{\sum_{i=1}^n w_i}$.
\end{proof}

\begin{replemma}{lem:distancecomp}[Distance comparison] Fix a point $x_0 \in \M$ and a system of privileged coordinates $z_1,\dots,z_n$ at $x_0$. Then there exists a positive constant $\theta(x_0) > 0$ such that for all $x$ with $\|x_0 - x\| \leq \theta(x_0)$,
\[
\|x_0 - x\| \leq d(x_0, x) \leq 2\Amax\|x_0 - x\|^{1/\rank}.
\]
\end{replemma}
\begin{proof}
The left-hand inequality is a consequence of the definition of $d$. For the right-hand inequality, we note for each coordinate that $|z_i(x)^{1/w_i}| = |y_i(x)|^{1/w_i} + O(\|y(x)\|^{2/w_i}) \leq \|x_0-x\|^{1/w_i} + O(\|x_0-x\|^{2/w_i})$, so then
$
d(x_0, x) \leq \Amax \|z(x)\|_{x_0} \leq 2\Amax\|x_0 - x\|^{1/\rank}
$
for small $\|x_0 - x\|$.
\end{proof}

\begin{replemma}{lem:ptubound} Let $x$ be a dynamically feasible trajectory and consider a partition of the time interval $0 = \tau_1 < \tau_2 < \dots < \tau_M = T$. Suppose that $\{y_m\} \subset \M$ satisfy a) $y_m \in B(x(\tau_m), \rho)$ for all $m \in \{1,\dots,M\}$, and b) more than a $(1-\alpha)$ fraction of the $y_m$ satisfy $y_m \in B(x(\tau_m), \beta \rho)$ for a parameter $\beta \in (0,1)$. Then the cost $c(y^*)$ of the trajectory $y^*$ sequentially connecting the nodes $y_1,\dots,y_M$ is upper bounded as
\[
c(y^*) \leq c(x) + 2M\rho(\beta + \alpha - \alpha\beta).
\]
\end{replemma}
\begin{proof}
Using the triangle inequality for the metric $d$ we compare:
{\small
\begin{align*}
c(y^*) &= \sum_{m=1}^{M-1} d(y_m, y_{m+1}) \leq \sum_{m=1}^{M-1} \bigl[d(y_m, x(\tau_m))\\
&\qquad  + d(x(\tau_m), x(\tau_{m+1}))+ d(x(\tau_{m+1}), y_{m+1})\bigr] \\
                    &= \sum_{m=1}^{M-1} d(x(\tau_m), x(\tau_{m+1})) + 2 \sum_{m=1}^{M} d(x(\tau_m), y_m) \\
                    &\qquad- d(x_{1}, y_{1}) - d(x_M, y_M)\\
                    &\leq c(x) \!+\! 2 \left(M \beta \rho + \alpha M (1\!-\!\beta)\rho\right)\leq c(x) \!+\! 2M\rho(\beta + \alpha - \alpha\beta).
\end{align*}
}
\end{proof}

\begin{replemma}{lem:samplesmall} Fix $n \in \N$, $\alpha \in (0,1)$, and let $S_1, \dots, S_M$ be disjoint subsets of $\Mfree$ with
\[
\mu(S_m) = \mu(S_1) \geq \left(\frac{2+\log(1/\alpha)}{n}\right)e^2 \mu(\Mfree),
\]
for each $m$. Let $V = \SF(n)$ and define
\[
K_n := \#\{m \in \{1,\dots,M\} : S_m \cap V = \emptyset\}.
\]
Then $\p{K_n \geq \alpha M} \leq \frac{e^{-\alpha  M}}{1-e^{-n}}$.
\end{replemma}
\begin{proof}
The proof relies on a Poissonization argument.  Let $\tilde{n}$ be a random variable drawn from a Poisson distribution with parameter $n/e^2$ (denoted as Poisson$(n/e^2)$). Consider the set of nodes $\widetilde{V} := \texttt{SampleFree}(\tilde{n})$. Let $\widetilde{K}_n$ be the Poissonized analogue of $K_n$, namely
\[
\widetilde{K}_n := \#\Bigl \{m \in \{1, \dots, M\}: S_m \cap \widetilde{V} = \emptyset \Bigr \}.
\]
From the definition of $\widetilde{V}$, we can see that $\p{K_n \ge \alpha  M)} = \probcond{\widetilde{K}_n \ge \alpha  M}{\tilde{n} = n}$.  Thus, we have
{\small
\begin{equation}\label{pois_ineq}
\begin{split}
  \p{\widetilde{K}_n \ge \alpha  M}  &= \sum_{j = 0}^{\infty} \probcond{\widetilde{K}_n \ge \alpha  M}{\tilde{n} = j} \, \p{\tilde{n} = j} \\
&  \ge \sum_{j = 0}^{n} \probcond{\widetilde{K}_n \ge \alpha  M}{ \tilde{n} = j}\,\p{\tilde{n} = j} \\
&  \ge \sum_{j = 0}^{n} \probcond{\widetilde{K}_n \ge \alpha  M}{ \tilde{n} = n} \, \p{\tilde{n} = j} \\
&  = \probcond{\widetilde{K}_n \ge \alpha  M}{ \tilde{n} = n} \, \p{\tilde{n} \le n} \\
& = \p{K_n \ge \alpha  M} \, \p{\tilde{n} \le n} \\
& \ge (1-e^{-n}) \p{K_n \ge \alpha  M},
\end{split}
\end{equation}
}
The third line follows from the fact that $\mathbb{P}(\widetilde{K}_n \ge \alpha  M | \tilde{n} = j)$ is nonincreasing in $j$, and the last line follows from a tail approximation of the Poisson distribution \citep[p. 17]{Penrose:03} and the fact that $\mathbb{E}[\tilde{n}] = n/e^2$.

The locations of the nodes in $\widetilde{V}$ are distributed as a spatial Poisson process with intensity $n / (e^2 \mu(\Mfree))$. This means that for a Lebesgue-measurable region $R \subseteq \Mfree$, the number of nodes in $R$ is distributed as a Poisson random variable with distribution Poisson$\Bigl ((n/e^2) \mu(R) / \mu(\Mfree)\Bigr)$, independent of the number of nodes in any region disjoint with $R$ \citep[Lemma 11]{Karaman.Frazzoli:IJRR2011}. Since by assumption the $S_m$ are all disjoint, we get that for $m \in \{1,\dots,M\}$,
\[
\p{S_m \cap \widetilde{V} = \emptyset} = e^{-(n/e^2) \mu(S_1) / \mu(\Mfree)}.
\]
Therefore, $\widetilde{K}_n$ is distributed according to a binomial distribution, in particular according to a distribution Binomial($M$, $e^{-(n/e^2) \mu(S_1) / \mu(\Mfree)}$). Noting that $e^{-(n/e^2) \mu(S_1) / \mu(\Mfree)} \leq \alpha/e^2$ by assumption, we have that $\mathbb{E}[\widetilde{K}_n] \leq \alpha  M/e^2$. So from a tail approximation to the Binomial distribution \citep[p. 16]{Penrose:03}, $\mathbb{P}(\widetilde{K}_n \ge \alpha M) \leq e^{-\alpha  M}$ and thus plugging into inequality \eqref{pois_ineq}
\[
\mathbb{P}(K_n \ge \alpha M) \leq \frac{e^{-\alpha  M}}{1-e^{-n}}.
\]
\end{proof}

\begin{replemma}{lem:samplebig} Fix $n \in \N$ and let $T_1, \dots, T_M$ be subsets of $\Mfree$, possibly overlapping, with
\[
\mu(T_m) = \mu(T_1) \geq \kappa\left(\log n/n \right)\mu(\Mfree)
\]
for each $m$ and some constant $\kappa > 0$. Let $V = \SF(n)$ and denote by $E_m$ the event that $T_m \cap V = \emptyset$ for each $m$. Then
\[
\p{\bigvee_{m=1}^M E_m} \leq Mn^{-\kappa}.
\]
\end{replemma}
\begin{proof}
We union bound the probability of any $E_m$ occurring:
\begin{align*}
\p{\bigvee_{m=1}^M E_m} &\leq \sum_{m=1}^M \p{T_m \cap V = \emptyset} \\
                        &= \sum_{m=1}^M \left(1-\frac{\mu(T_m)}{\mu(\Mfree)}\right)^n\\
                        &\leq M e^{-\left(\frac{\mu(T_m)}{\mu(\Mfree)}\right)n}\leq Mn^{-\kappa}.
\end{align*}
\end{proof}

\begin{reptheorem}{thm:DFMTccomp}[\DFMT cost comparison] Let $(\Sigma, \Mfree, \xinit, \Mgoal)$ be a  trajectory planning problem satisfying the assumptions $A_\Sigma$ and suppose $x: [0,T] \rightarrow \M$ is a feasible path with strong $\delta$-clearance, $\delta > 0$. Assume further that $x$ extends into the interior of $\Mgoal$, i.e. there exists $\gamma>0$ such that $B(x(T), \gamma) \subset \Mgoal$.
Let $c_n$ denote the cost of the path returned by \DFMT with $n$ vertices using a radius
\[
r_n = 4 \Amax (1 + \eta)^{1/D} \left(\frac{\mu(\Mfree)}{D}\right)^{1/D} \left(\frac{\log n}{n}\right)^{1/D}
\]
for a parameter $\eta \geq 0$. Then for fixed $\eps > 0$
\[
\p{c_n > (1+\eps)c(x)} = O(n^{-\eta/D}\log^{-1/D} n).
\]
\end{reptheorem}
\begin{proof}
Consider $n$ so that $r_n \leq \min\{\gamma, \delta\amin/2\Amax\}$, and apply Theorem~\ref{thm:pathtracing} to produce, with probability at least $1 - O(n^{-\eta/D}\log^{-1/D} n)$, a sequence of waypoints $\{y_m\}_{m=1}^{M} \subset V$ which $(\eps, r_n)$-trace $x$. We claim that in the event that such $\{y_m\}$ exist, the \DFMT algorithm will return a path with cost upper bounded as $c_n \leq c(y^*) \leq (1+\eps)c(x)$. It is clear that the desired result follows from this claim.

Assume the existence of $(\eps, r_n)$-tracing $\{y_m\}$. Note that our upper bound on $r_n$ implies that $\boxw(y_m, r_n/\amin)$ intersects no obstacles. This follows from the inclusion $\boxw(y_m, r_n/\amin) \subset B(y_m, \Amax r_n/\amin) \subset B^e(y_m, \Amax r_n/\amin)$, and the Euclidean distance bound
\begin{align*}
\inf_{a \in \Mobs} \|y_m - a\| &\geq \inf_{a \in \Mobs} \|x_m - a\| - \|y_m - x_m\|\\
		&\geq 2\left(\frac{\Amax}{\amin}\right)r_n - r_n \geq \left(\frac{\Amax}{\amin}\right)r_n.
\end{align*}
Consider running \DFMT to completion and for each $y_m$, let $c(y_m)$ denote the cost-to-come of $y_m$ in the generated tree. If $y_m$ is not contained in any edge of the tree, we set $c(y_m)=\infty$. We show by induction that
\begin{equation}\label{eqn:ctocclaim}
\min(c(y_m), c_n) \leq \sum_{k=1}^{m-1} d(y_k, y_{k+1}),
\end{equation}
for all $m \in \{2,\dots,M\}$.

The base case $m=2$ is trivial, since the first step in the \DFMT algorithm is to make every collision-free connection between $\xinit = y_1$ and the nodes contained in $\boxw(x, r_n/\amin)$, which will include $y_2$ and, thus, $c(y_2) = d(y_1, y_2)$. Now suppose \eqref{eqn:ctocclaim} holds for $m-1$; that means that one of the following four statements must hold.
\begin{itemize}
\item[1.] $c_n \le \sum_{k = 1}^{m-2}d(y_k, y_{k+1})$,
\item[2.] $c(y_{m-1}) \le \sum_{k = 1}^{m-2}d(y_k, y_{k+1})$ and \DFMT ends before considering $y_m$,
\item[3.] $c(y_{m-1}) \le \sum_{k = 1}^{m-2}d(y_k, y_{k+1})$ and $y_{m-1} \in H$ when $y_m$ is first considered,
\item[4.] $c(y_{m-1}) \le \sum_{k = 1}^{m-2}d(y_k, y_{k+1})$ and $y_{m-1} \notin H$ when $y_m$ is first considered.
\end{itemize}

Case 1: $c_n \le \sum_{k = 1}^{m-2}d(y_k, y_{k+1}) \le \sum_{k = 1}^{m-1}d(y_k, y_{k+1})$. \vskip0.5em

Case 2: $c(y_{m-1}) < \infty$ implies that $y_{m-1}$ enters $H$ at some point during \DFMT. The case that $y_m$ goes unconsidered means that the algorithm terminates before $x_{m-1}$ is ever minimum-cost element of $H$.  Since the end-node of the solution returned must have been the minimum-cost element of $H$, $c_n \le c(y_{m-1}) \leq \sum_{k = 1}^{m-2}d(y_k, y_{k+1}) \le \sum_{k = 1}^{m-1}d(y_k, y_{k+1})$. \vskip0.5em

Case 3: Since $\boxw(y_m, r_n/\amin)$ intersects no obstacles, $y_m$ must be connected to some parent when it is first considered. $y_{m-1}$ is a candidate, so $c(y_m) \le c(y_{m-1}) + d(y_{m-1}, y_{m}) \le \sum_{k = 1}^{m-1}d(y_k, y_{k+1})$. \vskip0.5em

Case 4: When $y_m$ is first considered, it is because there exists $z \in \boxw(y_m, r_n/\amin)$ such that $z$ is the minimum-cost element of $H$.  Again since $\boxw(y_m, r_n/\amin)$ intersects no obstacles and contains at least one node in $H$, $y_m$ must be connected to some parent when it is first considered. Since $c(y_{m-1}) < \infty$, there is a well-defined path $\mathcal{P} = \{v_1 = \xinit, v_2, \dots, v_q = y_{m-1}\}$ giving its cost-to-come in the \DFMT tree at termination.  Let $w = v_j$, where $j = \max_{i \in \{1,\dots,q\}} \{i : v_i \in H \text{ when } y_m \text{ is first considered}\}$. Then there are two subcases, either $w \in B(y_m, r_n)$ or $w \notin B(y_m, r_n)$.  If $w \in B(y_m, r_n)$, then,
\begin{align*}
c(y_m) & \le c(w) + d(w, y_m)\\&
 \le c(w) + d(w, y_{m-1}) + d(y_{m-1}, y_{m}) \\
& \leq c(y_{m-1}) + d(y_{m-1}, y_{m}) \le \sum_{k = 1}^{m-1}d(y_k, y_{k+1}).
\end{align*}

If $w \notin B(y_m, r_n)$, then $r_n \leq d(w, y_m) \leq d(w, y_{m-1}) + d(y_{m-1}, y_{m})$, so we have
\begin{align*}
c(y_m) & \le c(z) + d(z, y_m)  \le c(w) + r_n \\&\le c(w) + d(w, y_{m-1}) + d(y_{m-1}, y_{m})\\
		& \leq c(y_{m-1}) + d(y_{m-1}, y_{m}) \le \sum_{k = 1}^{m-1}d(y_k, y_{k+1}).
\end{align*}

Thus the inductive step holds in all cases and therefore \eqref{eqn:ctocclaim} holds for all $m$.  In particular taking $m = M$, and noting that $r_n \leq \gamma$ implies that $y_{M} \in \Mgoal$, this means that $c_n \leq c(y_M) \leq \sum_{k=1}^{M-1} d(y_k, y_{k+1}) = c(y^*)$ as desired.
\end{proof}

\end{document}